\documentclass{article}

\usepackage[preprint]{neurips_2020}

\usepackage[utf8]{inputenc} 
\usepackage[T1]{fontenc}    
\usepackage{hyperref}       
\usepackage{url}            
\usepackage{booktabs}       
\usepackage{amsfonts}       
\usepackage{nicefrac}       
\usepackage{microtype}      

\usepackage{graphicx}
\usepackage{bm,color,paralist}
\usepackage{algorithm,algpseudocode}
\usepackage{times}
\usepackage{subfig,multirow}
\usepackage{enumitem}
\usepackage{epstopdf}
\usepackage{setspace}
\usepackage{hyperref}
\usepackage{mathrsfs}

\usepackage{wrapfig}

\newtheorem{remark}{Remark}

\usepackage{amsmath}
\usepackage{subfig}

\newcommand{\blank}[1]{}
\newcommand{\ie}{{i.e.}}

\newcommand{\inred}[1]{\textcolor{red}{#1}}

\newcommand{\inblue}[1]{\textcolor{blue}{#1}}

\newcommand*{\Scale}[2][4]{\scalebox{#1}{$#2$}}
\newtheorem{theorem}{Theorem}

\newtheorem{lemma}{Lemma}

\newtheorem{problem}{Problem}
\newtheorem{assumption}{Assumption}
\newtheorem{proof}{Proof}
\newtheorem{definition}{Definition}

\title{Stable Prediction via Leveraging Seed Variable}

\author{%
  Kun Kuang \thanks{kunkuang@zju.edu.cn}\\
  Zhejiang University\\
\And
Bo Li  \\
  Tsinghua University \\
  \And
Peng Cui \\
  Tsinghua University \\
  \And
Yue Liu \\
  Peking University \\
  \And
Jianrong Tao \\
  Netease \\
      \And
Yueting Zhuang \\
   Zhejiang University \\
    \And
Fei Wu \\
   Zhejiang University \\
}

\begin{document}

\maketitle

\begin{abstract}
In this paper, we focus on the problem of stable prediction across unknown test data, where the test distribution is agnostic and might be totally different from the training one. In such a case, previous machine learning methods might exploit subtly spurious correlations in training data induced by non-causal variables for prediction. Those spurious correlations are changeable across data, leading to instability of prediction across data. By assuming the relationships between causal variables and response variable are invariant across data, to address this problem, we propose a conditional independence test based algorithm to separate those causal variables with a seed variable as priori, and adopt them for stable prediction. By assuming the independence between causal and non-causal variables, we show, both theoretically and with empirical experiments, that our algorithm can precisely separate causal and non-causal variables for stable prediction across test data. Extensive experiments on both synthetic and real-world datasets demonstrate that our algorithm outperforms state-of-the-art methods for stable prediction.
\end{abstract}

\section{Introduction}


Many machine learning algorithms have been shown to be very successful for prediction when the test
data have the same distribution as the training data. In real scenarios, however,
we cannot guarantee the unknown test data will have the
same distribution as the training data. For example, different
geographies, schools, or hospitals may draw from different
demographics, and the correlation structure among demographics
may also vary (e.g., one ethnic group may be more
or less disadvantaged in different geographies).
The model may exploit subtly genuine statistical relationships among predictors present in the training data to improve prediction, resulting in the instability of prediction across test data that out of training distribution.
Hence, how to learn a model for stable prediction across unknown test data is of paramount importance for both
academic research and practical applications.

To address the stable/invariant prediction problem, recently, many algorithms have been proposed, including domain generalization \cite{muandet2013domain}, causal transfer learning \cite{rojas2018invariant} and invariant causal prediction \cite{peters2016causal}. The motivation of these methods is to explore the invariant or stable structure between predictors and the response variable across multiple training data for stable prediction. But they cannot handle the test data whose distribution are out of all training environments.
Kuang et al. \cite{kuang2018stable,kuang2020stable} proposed to recover causation between predictors and response variable by global sample weighting, and separate causal variables for stable prediction. However, they either assume all predictors are binary or analyze based on linear model, which are impractical in real scenarios.

In the stable prediction problem \cite{kuang2018stable}, all predictors $\mathbf{X}$ can be separated into two categories, including causal variables $\mathbf{C}$ and non-causal variables $\mathbf{N}$, by whether it has causal effect on the response variable $Y$ or not, that is $\mathbf{X} = \{\mathbf{C},\mathbf{N}\}$.
For example, ears, noses, and whiskers are causal variables of cats to identify whether an image contains a cat or not, while the grass or other backgrounds are non-causal variables to recognize the cat.
Then, the generation of the response variable $Y$ can be denoted as $Y = f(\mathbf{X}) + \epsilon =  f(\mathbf{C}) + \epsilon$, where non-causal variables $\mathbf{N}$ should be independent with the response variable $Y$ conditional on the full sets of causal variables $\mathbf{C}$. But they might be spuriously correlated with either causal variables, response variable or both because of sample selection bias in data. For example, the variable ``grass'' would be spuriously correlated with label ``cat'' and become a powerful predictor if we select many images with ``cat on the grass'' as training data. Those spurious correlations between non-causal variables and the response variable are varied and unstable across datasets with different distributions, leading to unstable prediction across unknown test data. Hence, to address the stable prediction problem, one possible solution is to separate the causal and non-causal variables, and only adopt causal variables for model training and prediction. However, in practice, the analyst always have no prior knowledge on which variables are casual variables and which are non-causal variables.

Variable/Feature selection plays a very important role in machine learning filed. Traditional correlation based feature selection methods utilized either the correlation criteria \cite{nie2010efficient} or mutual information criteria \cite{peng2005feature} without distinguishing the spurious correlation, leading to unstable prediction across test data that out of training distribution.
In the literature of causality, causal discovery and causal estimation techniques can be adopted for causal variables selection.
PC \cite{spirtes2000causation}, FCI \cite{spirtes2000causation} and CPC \cite{ramsey2012adjacency} are three of the most prominent causal discovery methods based on conditional independence (CI) test, but their complexity grow exponentially with the number of variables. 
Moreover, PC method need assume causal sufficiency, i.e., the assumption that all common causes of observed predictors are observed.
\cite{athey2018approximate,kuang2017estimating} can approximately identify causal variables via estimating the causal effect of each variable, but they focused on binary predictors and required that all causal variables are observed.


\begin{wrapfigure}{t}{5.0cm}
\centering
  \includegraphics[width=2.1in]{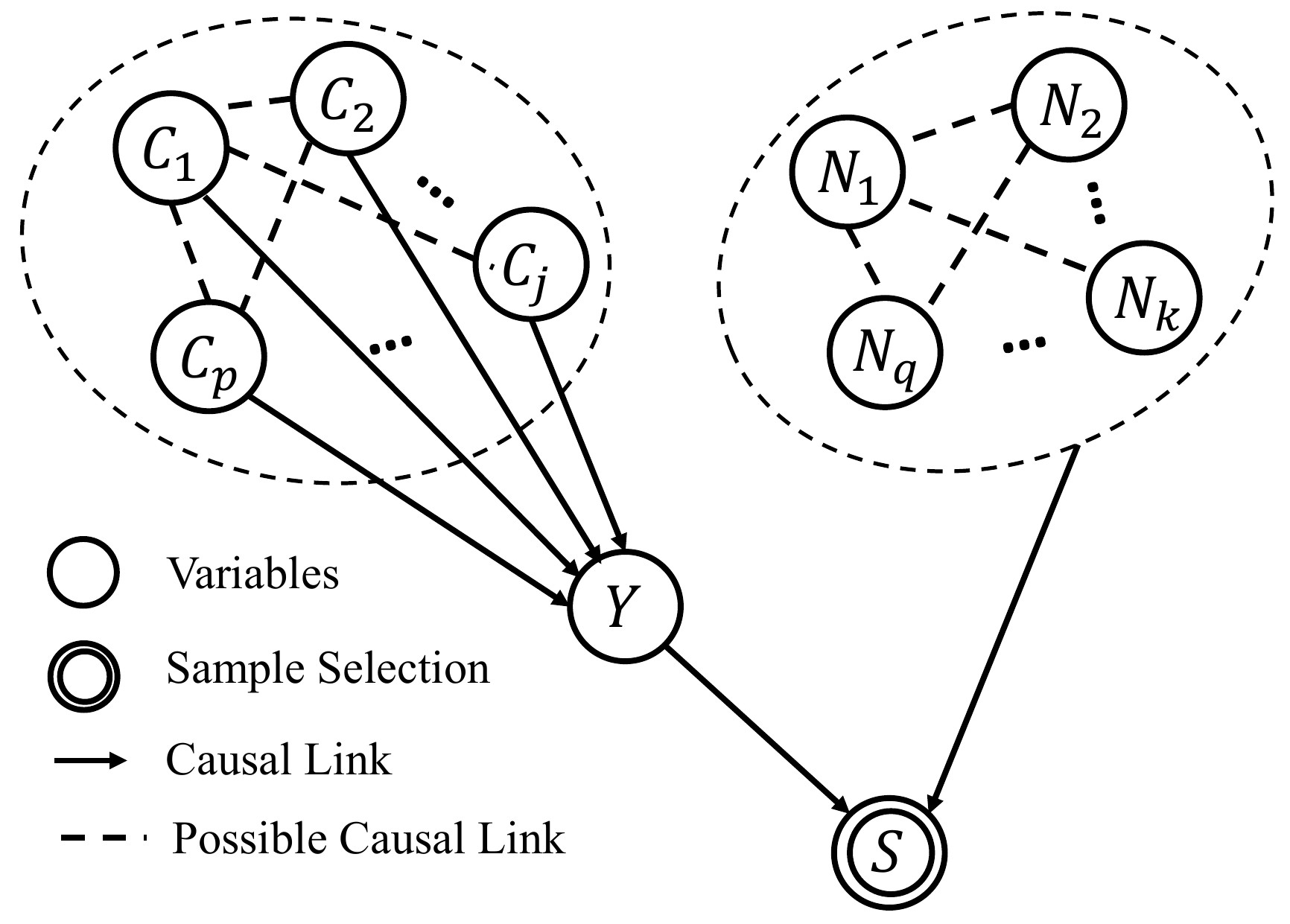}
\caption{SCM in our problem. Each causal variable $\mathbf{C}_i$ has a direct causal link to $Y$, but non-causal variable $\mathbf{N}_{j}$ does not. Under the sample selection \cite{bareinboim2012controlling} (indexed by variable $S$), some non-causal variables might be correlated with either response variable, causal variables, or both\protect\footnotemark[2].}
\label{fig:SCM_figure1}
\vspace{-0.1in}
\end{wrapfigure}
\footnotetext[2]{The distribution under sample selection is always conditioned on $S$.}

With considering the practical scenarios that causal sufficiency assumption is not met and parts of causal variables are unobserved or unmeasured,
in this paper, we propose a novel CI test based causal variable separation method for stable prediction.
By assuming that the set of causal variables $\mathbf{C}$ and non-causal variables $\mathbf{N}$ are independent, Fig. \ref{fig:SCM_figure1} illustrates the structural causal model (SCM) in our problem. Then, we provides a series of theorems to prove that one can separate the causal variables with a single CI test per variable. Specifically, as shown in Fig. \ref{fig:SCM_figure1}, if we know a seed variable $\mathbf{C}_0$ is one of the causal variables, then each causal variable $\mathbf{C}_{\cdot,i}$ should satisfy that $\mathbf{C}_{\cdot,i} \not\!\perp\!\!\!\perp \mathbf{C}_0 \mid Y$, and each non-causal variable $\mathbf{N}_{\cdot,j}$ should satisfy that $\mathbf{N}_{\cdot,j} \perp\!\!\!\perp \mathbf{C}_0 \mid Y$.
With those theoretical analyses, we present a CI test based causal variable separation method for stable prediction. At a first step, we apply our causal variable separation method on synthetic data, which leads to high precision on causal variable separation, and the precisely separated causal variables bring stability for prediction across unknown test data. In real-world applications, we also demonstrate that our algorithm outperforms baseline algorithms in both causal variable separation task and stable prediction task.

Comparing with previous CI based causal discovery methods \cite{spirtes2000causation,ramsey2012adjacency,buhlmann2010variable,yu2019causality}, our method do not rely on the assumption of causal sufficiency and remain unaffected even some causal variables are unobserved. Moreover, our algorithm separate the causal variables with a single CI test per variable, scaling algorithmic complexity from exponential to linear with the number of variables.
Comparing with sample based work on stable prediction \cite{kuang2018stable,kuang2020stable}, our method can be applied for continuous settings and separate the causal variables without assumptions on regression model.
Our work is similar with a recent paper \cite{mastakouri2019selecting}, which also adopt CI for causal variable selection. But the tailored problems are totaly different in the following ways: (i) \cite{mastakouri2019selecting} focused on detecting direct and indirect causes of a response variable under i.i.d settings, while our algorithm is designed for separating causal and non-causal variables under the biased settings with sample selection bias; (ii) \cite{mastakouri2019selecting} is tailored for the problem in which a cause variable of each candidate causal variable is known, while our algorithm assume the independence between causal and non-causal variables, and a seed variable as priori. Moreover, we applied our method to address agnostic distribution shift issue between training and unknown test data for stable prediction.

\section{Stable Prediction Problem}

Let $\mathcal{X}$, $\mathcal{Y}$ denote the space of observed predictors and response variable, respectively.
We define an \textbf{environment} $e\in \mathcal{E}$ to be a joint distribution $P_{\mathbf{X}Y}$ on $\mathcal{X} \times \mathcal{Y}$.
In practice, the joint distribution can vary across environments: $P^{e}_{XY} \neq P^{e'}_{XY}$ for $e,e'\in\mathcal{E}$.

In this paper, we consider a setting where a researcher has a single data set (data from one environment), and wishes to train a model that can then be applied to other environments. This type of problem might arise when a firm creates an algorithm that is then provided to other organizations to apply, for example, medical researchers might train a model and incorporate it in a software product that is used by a range of hospitals; academics might build a prediction model that is applied by governments in different locations.  The researcher may not have access to the end user's data for confidentiality reasons. The problem can be formalized as a stable prediction problem \cite{kuang2018stable} as follows:

\begin{problem}
(Stable Prediction). \textbf{Given} one training environment $e\in \mathcal{E}$ with dataset $\mathbf{D}^{e}=\{\mathbf{X}^{e},Y^{e}\}$, the task is to \textbf{learn} a predictive model that can \textbf{stably} predict across unknown test environments $\mathcal{E}$.
\end{problem}
In this problem, let $\mathbf{X} = \{\mathbf{C}, \mathbf{N}\}$, we define $\mathbf{C}$ as causal variables, and $\mathbf{N}$ as non-causal variables with the following assumption \cite{kuang2018stable}:
\begin{assumption}
\label{asmp:stable}
There exists a stable probability function P(y|c) such that for all environment $e\in \mathcal{E}$, $P(Y^e=y|\mathbf{C}^e=c, \mathbf{N}^e = n) = P(Y^e=y|\mathbf{C}^e=c) = P(y|c)$.
\end{assumption}

Thus, one can address the stable prediction problem by separating causal variables $\mathbf{C}$ and learning the stable function $P(y|c)$.
But, in practice, we have no prior knowledge on which variables are causal and which are non-causal.
In this work, we focus on stable prediction via separating causal variables.
\begin{assumption}
\label{asmp:independent}
Causal variables $\mathbf{C}$ and non-causal variables $\mathbf{N}$ are independent. Formally, $\mathbf{C}\perp\!\!\!\perp \mathbf{N}$.
\end{assumption}

Assumption \ref{asmp:stable} and \ref{asmp:independent} illuminate that the non-causal variable is independent with response variable during the data generation processing (\ie, $Y = f(\mathbf{X}) + \epsilon = f(\mathbf{C})+\epsilon$), but it might be spuriously correlated with either response variable, causal variables, or both since sample selection bias problem as shown in Fig. \ref{fig:SCM_figure1}. These spurious correlations might vary across environments. Hence, to make a stable prediction, one should guarantee the prediction only depending on the causal variables.


\section{Methods}

\subsection{Background on Causal Graph}
Firstly, we revisit key concepts and theorems related to $d$-separation and CI in causal graph.

Let $G = \{\mathbf{V}, E\}$ represents a causal directed acyclic graph (DAG) with nodes $\mathbf{V}$ and edges $E$, where a node denotes a variable and an edge represents the direct dependence or causal direction between two variables. In a DAG, $\mathbf{V}_i \rightarrow \mathbf{V}_j$ refers to that $\mathbf{V}_i$ is a cause of $\mathbf{V}_j$ and $\mathbf{V}_j$ is an effect of $\mathbf{V}_i$.

\begin{definition}[$d$-separation \cite{pearl2009causality}]\label{def:D_separation}
In a DAG $G$, a path $\pi$ is said to be $d$-separated by a set of nodes $\mathbf{Z}$ if and only if (i) $\pi$ contains a chain $\mathbf{V}_i \rightarrow \mathbf{V}_k \rightarrow \mathbf{V}_j$ or a fork $\mathbf{V}_i \leftarrow \mathbf{V}_k \rightarrow \mathbf{V}_j$ such that the middle node $\mathbf{V}_k$ is in $Z$, or (ii) $\pi$ contains a collider $\mathbf{V}_i \rightarrow \mathbf{V}_k \leftarrow \mathbf{V}_j$ such that the middle node $\mathbf{V}_k$ is not in $Z$ and such that no descendant of $\mathbf{V}_k$ is in $Z$.
\end{definition}

\begin{definition}[Conditional Independence] Given two distinct variables $\mathbf{V}_i,\mathbf{V}_j \in \mathbf{V}$ are said to be conditionally independent given a subset of variables $Z \subseteq \mathbf{V} \setminus \{\mathbf{V}_i, \mathbf{V}_j\}$ (i.e. $\mathbf{V}_i \perp\!\!\!\perp \mathbf{V}_j | \mathbf{Z})$, if and only if $P(\mathbf{V}_i,\mathbf{V}_j|Z) = P(\mathbf{V}_i|Z)P(\mathbf{V}_j|Z)$. Otherwise, $\mathbf{V}_i$ and $\mathbf{V}_j$ are conditionally dependent given $\mathbf{Z}$ (i.e. $\mathbf{V}_i \not\!\perp\!\!\!\perp\mathbf{V}_j | \mathbf{Z})$).
\end{definition}

The connection between $d$-separation and CI is established through the following lemma:
\begin{lemma}[Probabilistic Implications of $d$-Separation \cite{geiger1990identifying,pearl2009causality}] \label{theo:connection_d_ci}
If variables $\mathbf{V}_i$ and $\mathbf{V}_j$ are d-separated by $\mathbf{Z}$ in a DAG $G$, then $\mathbf{V}_i$ is independent of $\mathbf{V}_j$ conditional on $\mathbf{Z}$ in every distribution compatible with the DAG $G$. Conversely, if $\mathbf{V}_i$ and $\mathbf{V}_j$ are not d-separated by $\mathbf{Z}$ in a DAG $G$, then $\mathbf{V}_i$ and $\mathbf{V}_j$ are dependent conditional on $\mathbf{Z}$ in at least one distribution compatible with $G$.
\end{lemma}

\subsection{Causal Variables Separation}
Based on lemma \ref{theo:connection_d_ci}, in this paper, we propose an elaborative but effective causal variables separation algorithm by combining the mechanisms of $d$-separation and causality with the following assumption.

\begin{assumption}
\label{asmp:one_CF}
We have prior knowledge on one causal variable. Formally, we know $\mathbf{C}_0 \in \mathbf{C}$.
\end{assumption}

Under assumption \ref{asmp:one_CF}, we have the following theorem to support for precisely separating the set of causal and non-causal variables. Then, the set of causal variables can be applied for stable prediction.

\begin{figure}[t]
\centering
\vspace{-0.1in}
\subfloat[Path between $\mathbf{C}_0$ and $\mathbf{N}_i$ \label{fig:path_cn}]{
  \includegraphics[width=2.0in]{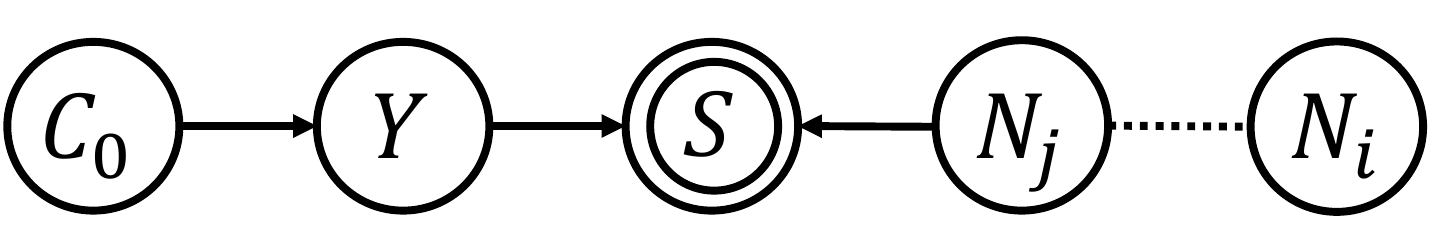}
}
\ \ \ \
\subfloat[Path between $\mathbf{C}_0$ and $\mathbf{C}_i$ \label{fig:path_cc}]{
  \includegraphics[width=1.3in]{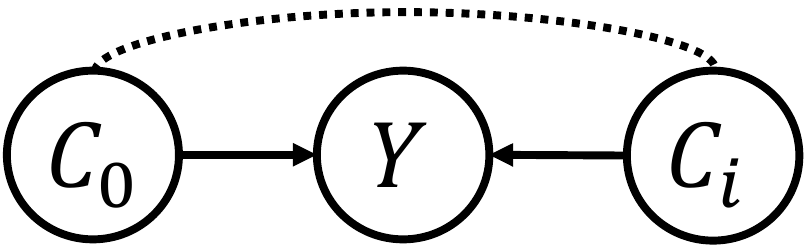}
}
\vspace{-0.2in}
\caption{Causal paths between a known causal variable $\mathbf{C}_0$ and other variables, including $\mathbf{C}_i$ and $\mathbf{N}_i$. The dash line between two variables refers to the causal link/path between them is unknown.}
\label{fig:path_cc_cn}
\vspace{-0.1in}
\end{figure}

\begin{theorem}\label{theo:causal_features_selection}
Given a causal variable $\mathbf{C}_0$, observed variables $\mathbf{X}$ and response variable $Y$, and assuming \ref{asmp:stable}\&\ref{asmp:independent}\&\ref{asmp:one_CF}, if $\mathbf{X}_{i} \not\!\perp\!\!\!\perp \mathbf{C}_0 \mid Y$, then $\mathbf{X}_{i}$ belongs to the set of causal variables, otherwise, it belongs to the set of non-causal variables.
\end{theorem}
\begin{proof}
Assumption \ref{asmp:stable} implies that non-causal variables $\mathbf{N}$ are not direct causes of response $Y$, but causal variables $\mathbf{C}$ are the direct causes. Hence, in our causal DAG, there exists a direct edge from each causal variable $\mathbf{C}_i$ to response $Y$, but $\mathbf{N}$ have no any edges that directly point to $Y$.
Assumption \ref{asmp:independent} guarantees no causal link between any causal and non-causal variables, but the causal structure among causal variables (or non-causal variables) might be very complex and unknown.
With considering the sample selection bias is generated based on the response $Y$ and part of non-causal variables $\mathbf{N}$, the causal DAG in our problem is shown in Fig. \ref{fig:SCM_figure1}.

From Fig. \ref{fig:SCM_figure1}, the path between the seed causal variable $\mathbf{C}_0$ and any non-causal variable $\mathbf{N}_i$ can be represented as Fig. \ref{fig:path_cn}, where the causal links between $\mathbf{N}_i$ and $\mathbf{N}_j$ are unknown, could be very complex or could $\mathbf{N}_j$ is exactly $\mathbf{N}_i$ if sample selection is based on $\mathbf{N}_i$ and $Y$. With the definition of $d$-separation, we have that $\mathbf{C}_0$ and $\mathbf{N}_i$ are $d$-separated by variable $Y$. Hence, $\mathbf{N}_i \perp\!\!\!\perp \mathbf{C}_0 \mid Y$ for any $\mathbf{N}_i \in \mathbf{N}$ guaranteed by the lemma \ref{theo:connection_d_ci}.

On the other hand, the path between the seed causal variable $\mathbf{C}_0$ and any other causal variable $\mathbf{C}_i$ can be represented as Fig. \ref{fig:path_cc}, where the causal links between $\mathbf{C}_0$ and $\mathbf{C}_i$ are unknown. Similarity, with the definition of $d$-separation, we know that the response variable $Y$ is a collider and cannot $d$-separate $\mathbf{C}_0$ and $\mathbf{C}_i$. Therefore, with the lemma \ref{theo:connection_d_ci}, we have $\mathbf{C}_i \not\!\perp\!\!\!\perp \mathbf{C}_0 \mid Y$ for any $\mathbf{C}_i \in \mathbf{C}$.

\begin{algorithm}[t]
\caption{{top-$k$ Causal Variables Separation/Selection}}
\label{alg:CFS}
\begin{algorithmic}[1]
\Require
$\mathbf{X} \in \mathbb{R}^{n\times p}$, $Y \in \mathbb{R}^{n}$, $\mathbf{C}_{0}$ and parameter $k$
\Ensure
top-$k$ casual variables
\For{each variable $\mathbf{X}_{i} \in \mathbf{X}$}
\State Calculate p-value of CI test: $pv_i=\text{CI-test}(\mathbf{X}_{i}\perp\!\!\!\perp \mathbf{C}_{0} \mid Y)$
\EndFor
\State $\mathbf{X}_{ranking} = Ranking(\mathbf{X}, \mathbf{pv})$ \Comment{Ranking $\mathbf{X}_{i} \in \mathbf{X}$ by their p-value $pv_i$ in ascending order}\\
\Return top-$k$ ranked variables in $\mathbf{X}_{ranking}$
\end{algorithmic}
\end{algorithm}

Overall, we can separate causal and non-causal variables by a single CI test per variable, and $\mathbf{X}_{i}$ belongs to the set of causal variables if $\mathbf{X}_{i} \not\!\perp\!\!\!\perp \mathbf{C}_0 \mid Y$, otherwise, $\mathbf{X}_{i}$ is non-causal variable.
\end{proof}

Based on theorem \ref{theo:causal_features_selection}, we propose a causal variable separation algorithm via one single CI test per variable. The details of our algorithm are summarized in Algorithm \ref{alg:CFS}.
With the separated top-$k$ causal variables, we can learn a predictive model for stable prediction.

\begin{remark}
From the proof of theorem \ref{theo:causal_features_selection}, we know that to identify whether a variable is causal or not, our algorithm only need a single CI test of that variable and a known causal variable conditional on the response variable, with no need to know the other causal variables or common causes of observed variables. Then, we conclude that (i) our algorithm is not affected by the unobserved causal variables, but missing some causal variables would decrease the performance of predictive model on prediction; and (ii) the causal sufficiency assumption is not necessary for our algorithm, but we need to assume the independence between causal and non-causal variables.
\end{remark}

\textbf{Complexity Analysis.}
Note that our algorithm requires only a single CI test per variable. Therefore, it speeds up the causal variables separation as it scales linearly with the number of variables, hence its complexity is $O(cp)$, where $p$ is the dimension of observed variables and $c$ is a constant denoting the complexity of a single CI test.

\textbf{Discussions on assumptions.}
Assumption \ref{asmp:stable} refers to that the underlying predictive mechanism is invariant across environments, which is the basic assumption for causal variables identification and stable/invariant prediction \cite{peters2016causal,kuang2018stable}.
In assumption \ref{asmp:independent}, we assume the independence between causal variables and non-causal variables, which is critical to our method. In practice, however, one might adopt disentangled representation \cite{thomas2018disentangling} or orthogonal techniques \cite{ahmed2012orthogonal} to guarantee this assumption to be satisfied on feature representation space. We leave this in future work.
As for assumption \ref{asmp:one_CF}, we think it is reasonable and acceptable in real applications. For example, if we want to predict the crime rate, we could know the income is one causal variable. Moreover, one can identify a causal variable as seed variable by estimating its causal effect \cite{athey2018approximate,kuang2017estimating}.

\section{Experiments}


\subsection{Baselines}
We implement the following variable selection methods as baselines, (i) correlation based methods, including minimal Redundancy Maximal Relevance (mRMR) \cite{peng2005feature}, Random Forest (RF) \cite{breiman2001random} and LASSO \cite{tibshirani1996regression}, they would be affected by the spurious correlation between non-causal variable and the response variable, and select non-causal variables for prediction; (ii) causation based methods, including PC-simple\footnote{Previous CI based methods either need observe all causal variables, or assume causal sufficiency, moreover, with curse of dimensionality. So, we only compare with PC-simple, a prominent CI based method.} \cite{buhlmann2010variable} and causal effect (CE) estimator \cite{athey2018approximate,kuang2017estimating}, they need to assume all causal variable are observed, moreover, PC-simple requires causal sufficiency and with curse of dimensionality; (iii) stable/invariant learning based methods, including invariant causal prediction (ICP) \cite{peters2016causal} and global balancing algorithm (GBA) \cite{kuang2018stable,kuang2020stable}, ICP need multiple training environments for reveal causation and GBA requires tremendous training data for global sample weighting.

In our algorithm, we employ causal effect estimator \cite{kuang2017estimating} to identify one causal variable without assumption \ref{asmp:one_CF}. Then, we execute CI test with bnlearn method \cite{scutari2009learning}, denoted as \emph{Our+BNCI}, and RCIT \cite{strobl2019approximate} method, denoted as \emph{Our+RCIT}.

We do not compare with a recent causal variable selection method \cite{mastakouri2019selecting}, since it requires the knowledge of a cause variable of each candidate causal variable, which is not applicable in our problem.

ICP method cannot be applied for variables ranking, but selecting a subset of variables for prediction, where the size of that subset variables is determined by its algorithm. Hence, the experimental results of ICP reported in this paper is based on its unique subset of selected variables.

Based on the selected variables from each algorithm, we apply a linear model\footnote{For simplification, we use linear model to evaluate the selected variables, other models can also be applied.} for prediction to check their stability across unknown test data.

\subsection{Evaluation Metrics}
To evaluate the performance of causal variable separation/selection, we use precision@k and ranking index of unstable non-causal variable as evaluation metrics.
Precision@k refers to the proportion of top-k selected variables that are hitting the true causal variables set as follows:
\begin{eqnarray}
Precision@k = \frac{|\{x_i|x_i\in \hat{\mathbf{C}}, index(x_i)<k, x_i\in \mathbf{C}\}|}{k},
\end{eqnarray}
where $\hat{\mathbf{C}}$ and $\mathbf{C}$ refer to the set of selected causal variables and true causal variables, respectively. $index(x_i)$ is the ranking index of variable $x_i$ in the selected variables $\hat{\mathbf{C}}$.

Similar to \cite{kuang2018stable}, we also adopt Average\_Error and Stability\_Error to measure the performance of stable prediction with the following definition:
\begin{align}
\label{metrics:acc_and_stb}
\Scale[0.85]{\mbox{Average\_Error}} = \Scale[0.85]{ \frac{1}{|\mathcal{E}|}\sum\limits_{e \in \mathcal{E}}\mbox{RMSE}(\mathbf{D}^e)},\quad
\Scale[0.85]{\mbox{Stability\_Error}}=
\Scale[0.85]{\sqrt{\frac{1}{|\mathcal{E}|-1}\sum\limits_{e \in \mathcal{E}}\left(\mbox{RMSE}(\mathbf{D}^e)-\mbox{Average\_Error}\right)^{2}}}.
\end{align}


\subsection{Experiments on Synthetic Data}


\subsubsection{Dataset}
To generate the synthetic datasets, we consider the sample size $n = 2000$ and dimension of observed variables $p = \{10,20,40,80\}$. We first generate the observed variables $\mathbf{X} = \{\mathbf{C},\mathbf{N}\}$. From Fig. \ref{fig:SCM_figure1} and assumption \ref{asmp:independent}, we know causal variables $\mathbf{C}$ and non-causal variables $\mathbf{N}$ should be independent, but the causal variables $\mathbf{C}$ could be dependent with each other, and the same to non-causal variables $\mathbf{N}$. Hence, we generate $\mathbf{X} = \{\mathbf{C}_{,1},\cdots,\mathbf{C}_{,p_c},\mathbf{N}_{,1},\cdots,\mathbf{N}_{,p_n}\}$ with the help of auxiliary variables $\mathbf{Z}_{\mathbf{C}}$ and $\mathbf{Z}_{\mathbf{N}}$ with independent Gaussian distributions as:
\begin{eqnarray}
\mathbf{Z_C}_{,1}, \cdots, \mathbf{Z_C}_{,p}\   \overset{iid}{\sim} \mathcal{N}(0,1);\ \ \ \ \ \ \mathbf{C}_{,i} = 0.8*\mathbf{Z_C}_{,i} + 0.2*\mathbf{Z_C}_{,i+1}, \,\,  i = 1, 2, \cdots, p_c\\
\mathbf{Z_N}_{,1}, \cdots, \mathbf{Z_N}_{,p}\   \overset{iid}{\sim} \mathcal{N}(0,1);\ \ \ \ \ \ \mathbf{N}_{,j} = 0.8*\mathbf{Z_N}_{,j} + 0.2*\mathbf{Z_N}_{,j+1}, \,\,  j = 1, 2, \cdots, p_n,
\end{eqnarray}
where the number of causal variables $p_c = 0.3*p$ and the number of non-causal variables $p_n = 0.7*p$. $\mathbf{C}_{,i}$ and $\mathbf{N}_{,j}$ represent the $i^{th}$ and $j^{th}$ variable in $\mathbf{C}$ and $\mathbf{N}$, respectively.

Then, we generate the response variable $Y$ as:
\begin{eqnarray}
\Scale[1.0]{Y = \sum_{i=1}^{ps} \alpha_i\cdot \mathbf{C}_{,i} + \sum_{j=1}^{p_c}\beta_j\cdot e^{\mathbf{C}_{\cdot,j}\mathbf{C}_{\cdot,j+1}\mathbf{C}_{\cdot,j+2}} + \varepsilon,}
\end{eqnarray}
where $\alpha_i = (-1)^{i}\cdot p_c/i$, $\beta_j = I(mod(j,3)\equiv1)$ and $\varepsilon =  \mathcal{N}(0,0.3)$.
The $I(\cdot)$ is the indicator function and function $mod(x, y)$
returns the modulus after division of $x$ by $y$.

From the generation of $Y$, we know that $Y$ is only affected by the causal variables $\mathbf{C}$, and independent with the non-causal variables $\mathbf{N}$. In real applications, however, some non-causal variables might be spuriously correlated with $Y$ since sample selection bias as shown in Fig. \ref{fig:SCM_figure1}, and their correlation might vary across datasets.
To check the stability of algorithms under that practical setting, we generate a set of environments, each with a stable probability $P(Y|\mathbf{C})$, but a distinct spuriously correlation $P(Y|\mathbf{N})$. For simplification, we only set one non-causal variable $\mathbf{N}_{,pn}$ as the \emph{unstable non-causal variable}, and change its spuriously correlation $P(Y|\mathbf{N}_{,pn})$ across environments.

Specifically, we vary $P(Y|\mathbf{N}_{,pn})$ via biased sample selection with a bias rate $r\in[-3,-1)\cup(1,3]$ based on $\mathbf{N}_{,pn}$ and $Y$ as shown in Fig. \ref{fig:SCM_figure1}.
For each sample, we select it with probability $Pr = |r|^{-5*D_i}$, where $D_i = |Y-sign(r)*\mathbf{N}_{,pn}|$.
If $r>0$, $sign(r) = 1$; otherwise, $sign(r) = -1$.

Note that $r>1$ corresponds to positive spurious correlation between $Y$ and $\mathbf{N}_{,pn}$, while $r<-1$ refers to the negative spurious correlation between $Y$ and $\mathbf{N}_{,pn}$.
The higher value of $|r|$, the stronger correlation between $\mathbf{N}_{,pn}$ and $Y$.  Different value of $r$ refers to different environments.
All methods are trained with $r_{train}=2.0$, but tested across environments with different $r_{test} \in[-3,-1)\cup(1,3]$.

\begin{minipage}{\textwidth}
        \begin{minipage}[t]{0.45\textwidth}
            \centering
            \makeatletter\def\@captype{table}\makeatother\caption{Results of precision@k, where $k$ equals the number of causal variable, namely $k = p*0.3$. ICP method cannot be applied for selecting variable with specific size.}
            \label{tab:precision@k}
\resizebox{!}{1.7cm}
{
\begin{tabular}{|c|c|c|c|c|}
\hline
Dimension  & p=10  &  p=20  &  p=40  &  p=80  \\
\hline
\hline
mRMR & 0.333 & 0.167 & 0.167 & 0.167\\ \hline
RF & 0.667 & 0.500 & 0.250 & 0.333\\ \hline
LASSO & 0.667 & 0.833 & 0.500 & 0.125\\ \hline
PC-simple & 0.667 & 0.667 & 0.250 & 0.167\\ \hline
CE & \textbf{1.000} & 0.833 & 0.917 & 0.792\\ \hline
ICP & - & - & - & -\\ \hline
GBA & \textbf{1.000} & 0.833 & 0.917 & 0.583\\ \hline
Our+BNCI & \textbf{1.000} & \textbf{1.000} & \textbf{1.000} & \textbf{0.833}\\ 
Our+RCIT & \textbf{1.000} & \textbf{1.000} & 0.917 & 0.750\\ \hline
\end{tabular}
}
        \end{minipage}
        \ \ \ \ \ \ \ \ \ \ \
        \begin{minipage}[t]{0.45\textwidth}
        \centering
        \makeatletter\def\@captype{table}\makeatother\caption{Ranking index of the unstable non-causal variable $\mathbf{N}_{,pn}$, where ``Y'' denotes that the unstable non-causal variable is in the selected subset in ICP method.}
\label{tab:unstable_noncausal_feature_ranking}
\resizebox{!}{1.7cm}
{
\begin{tabular}{|c|c|c|c|c|}
\hline
Dimension  & p=10  &  p=20  &  p=40  &  p=80  \\
\hline
\hline
mRMR & 1 & 1 & 1 & 1\\ \hline
RF & 1 & 1 & 1 & 1\\ \hline
LASSO & 3 & 1 & 1 & 1\\ \hline
PC-simple & 1 & 1 & 1 & 1\\ \hline
CE & 4 & 2 & 2 & 3\\ \hline
ICP & Y & Y & Y & Y\\ \hline
GBA & 4 & 2 & 3 & 1\\ \hline
Our+BNCI & \textbf{5} & \textbf{14} & 16 & \textbf{77}\\ 
Our+RCIT & 4 & 7 & \textbf{25} & 65\\ \hline

\end{tabular}
}
        \end{minipage}
    \end{minipage}

\subsubsection{Results}


\textbf{Results on Causal Variables Separation/Selection.} We report the results on causal variable selection from two aspects, including the ranking of causal variable with precision$@$k in Tab.\ref{tab:precision@k}  and ranking of unstable non-causal variable in Tab. \ref{tab:unstable_noncausal_feature_ranking}. The ranking of causal variables determines the average error of prediction across environments, the closer to 1 of precision$@$k, the better; while the ranking of unstable non-causal variable determines the stability error of prediction across environments, the lower ranking, the better.
From Tab. \ref{tab:precision@k} and \ref{tab:unstable_noncausal_feature_ranking}, we conclude that: (i) Traditional correlation based variables selection methods, including mRMR, Random Forest and LASSO cannot precisely select the causal variables (with lower precision$@$k) and rank the unstable non-causal variable with a higher ranking. The main reason is that the spurious correlation is more significant than causation under the sample selection bias. (ii) The performance of PC-simple is similar to correlation based method, since it's hard to search the optimal solution for PC-simple via naively random search, moreover, it relies on the causal sufficiency assumption and needs to observed all causal variables. (iii) The performance of causation based methods, including CE and GBA, is better than those correlation based methods with higher precision$@$k and lower ranking of unstable non-causal variable. Since by revealing part of causations among variables, they can reduce spurious correlations in training data. But their performances are still worse than our methods in high dimensional settings, since they need enough training data for a better sample rewighting, moreover, they need to observed all causal variables. (iv) Our methods achieve the best performance for the separation/selection of causal variables (with highest precision$@$k) and the ranking of unstable non-causal variable.

\begin{figure}[t]
\centering
\includegraphics[width=1.6in]{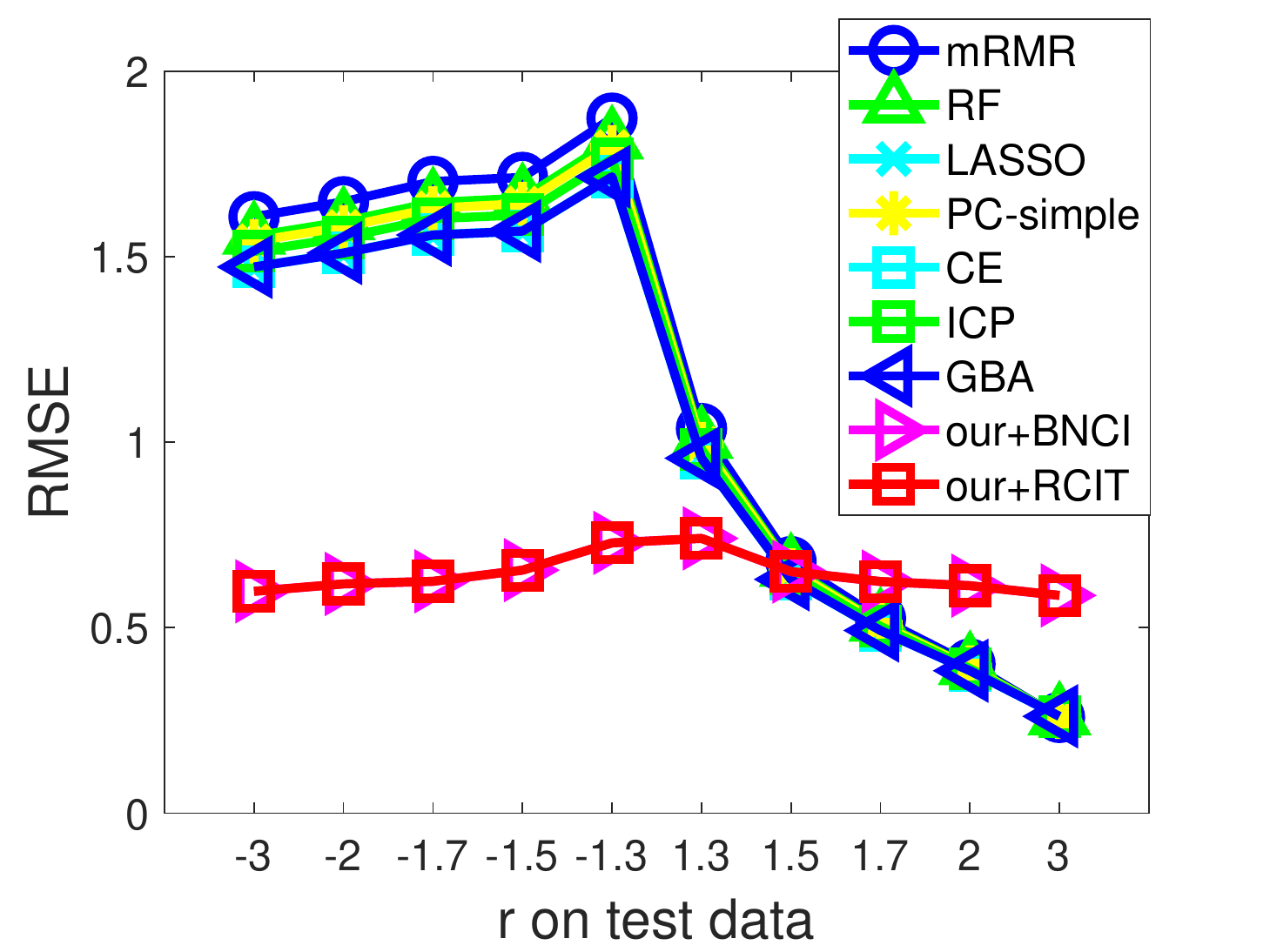}
\vspace{-0.1in}
\caption{Prediction results across unknown test data with $n=2000, p=20$. All methods are trained with $r_{train}=2.0$, but tested across environments with different $r_{test} \in[-3,-1)\cup(1,3]$.}
\label{fig:results_synthetic}
\end{figure}

\begin{table*}[]
\centering
\caption{Results of Average\_Error and Stability\_Error with different dimension $p$.}
\vspace{-0.1in}
\label{tab:AE_SE}
\resizebox{!}{1.45cm}
{
\begin{tabular}{|c|c|c|c|c|c|c|c|c|}
\hline
Dimension     & \multicolumn{2}{c|}{p=10}         & \multicolumn{2}{c|}{p=20}         & \multicolumn{2}{c|}{p=40}         & \multicolumn{2}{c|}{p=80}         \\ \hline
Metrics       & Average\_Error & Stability\_Error & Average\_Error & Stability\_Error & Average\_Error & Stability\_Error & Average\_Error & Stability\_Error \\ \hline
\hline
mRMR & 1.058 & 0.548 & 1.145 & 0.599 & 1.179 & 0.625 & 1.177 & 0.619\\ \hline
RF & 0.994 & 0.506 & 1.110 & 0.576 & 1.174 & 0.622 & 1.177 & 0.619\\ \hline
LASSO & 0.994 & 0.506 & 1.055 & 0.541 & 1.170 & 0.618 & 1.177 & 0.619\\ \hline
PC-simple & 1.039 & 0.536 & 1.100 & 0.570 & 1.175 & 0.622 & 1.178 & 0.619\\ \hline
CE & \textbf{0.413} & \textbf{0.019} & 1.055 & 0.541 & 1.132 & 0.593 & 1.168 & 0.613\\ \hline
ICP & 0.680 & 0.313 & 1.082 & 0.558 & 1.172 & 0.621 & 1.176 & 0.620\\ \hline
GBA & \textbf{0.413} & \textbf{0.019} & 1.055 & 0.541 & 1.132 & 0.594 & 1.167 & 0.612\\ \hline
Our+BNCI & \textbf{0.413} & \textbf{0.019} & \textbf{0.644} & \textbf{0.049} & \textbf{0.879} & \textbf{0.111} & \textbf{1.017} & \textbf{0.160}\\ 
Our+RCIT & \textbf{0.413} & \textbf{0.019} & \textbf{0.644} & \textbf{0.049} & 0.909 & 0.121 & 1.020 & 0.161\\ \hline

\hline
\end{tabular}
}
\vspace{-0.2in}
\end{table*}

\textbf{Results on Stable Prediction.}
With the variable ranking list form each algorithm, we select top-$k$ ranked variables to evaluate their performances on stable prediction across unknown test environments, where $k$ is set as the number of causal variables (\ie, $k=0.3*p$). Fig. \ref{fig:results_synthetic} and Tab. \ref{tab:AE_SE} demonstrate the experimental results on stable prediction.
From Fig. \ref{fig:results_synthetic}, we find that (i) the performance of our methods are worse than baselines when $r_{test}>1.5$. This is because the spurious correlation between unstable non-causal variable and the response variable are highly similar between training data ($r_{train} = 2.0$) and test data when $r_{test}>1.5$, and that correlation can be exploited for improving predictive performance; (ii) the performance of our methods are much better than baseline when $r_{test}<-1.3$, where that spurious correlation are totaly different between training ($r_{train} = 2.0$) and test data $r_{test}<-1.3$, leading to unstable prediction on baselines; (iii) our methods achieve the most stable prediction across all test data, since our algorithm can precisely separate the causal variables and achieve the lowest ranking of unstable non-causal variable as reported in Tab.\ref{tab:precision@k} and Tab. \ref{tab:unstable_noncausal_feature_ranking}.

To clearly demonstrate the advantages of our algorithm on stable prediction, we report the detail results under different synthetic settings in Tab. \ref{tab:AE_SE}. From the results, we can conclude that our algorithm can make stable prediction across unknown environments via causal variable separation.

\subsection{Experiments on Real-World Data}

\par \textbf{Dataset.}
To evaluate the performance of our algorithm in real-world datasets, we apply it to a Parkinson's telemonitoring dataset\footnote{\url{https://archive.ics.uci.edu/ml/datasets/parkinsons+telemonitoring}}, which was wildly used for the problem of domain generalization \cite{muandet2013domain,blanchard2017domain} and other regression tasks \cite{tsanas2009accurate}.
This dataset consists of biomedical voice measurements from 42 patients with early-stage Parkinson's disease recruited for a six-month trial of a telemonitoring device for remote symptom progression monitoring.
For each patient, there are about 200 recordings, which were automatically recorded in the patients' home.
The task is to predict the clinician's motor UPDRS scoring of Parkinson's disease symptoms from patients' features, including their age, gender, test time and many other measures.

\begin{figure*}[t]
\centering
\subfloat[RMSE on env. G1 \label{fig:real_RMSE_on_Env1}]{
  \includegraphics[width=1.3in]{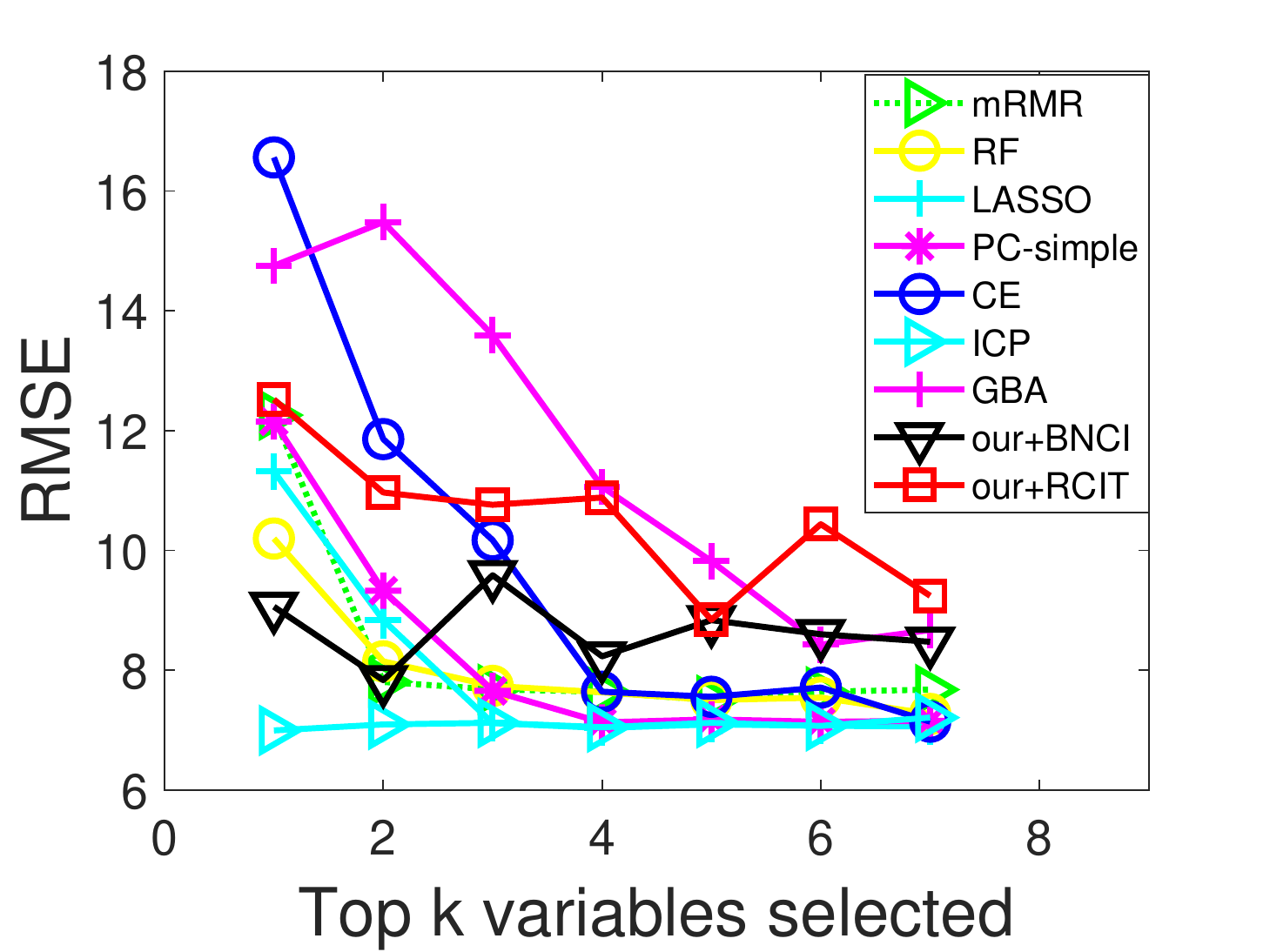}
}
\subfloat[RMSE on env. G2 \label{fig:real_RMSE_on_Env2}]{
  \includegraphics[width=1.3in]{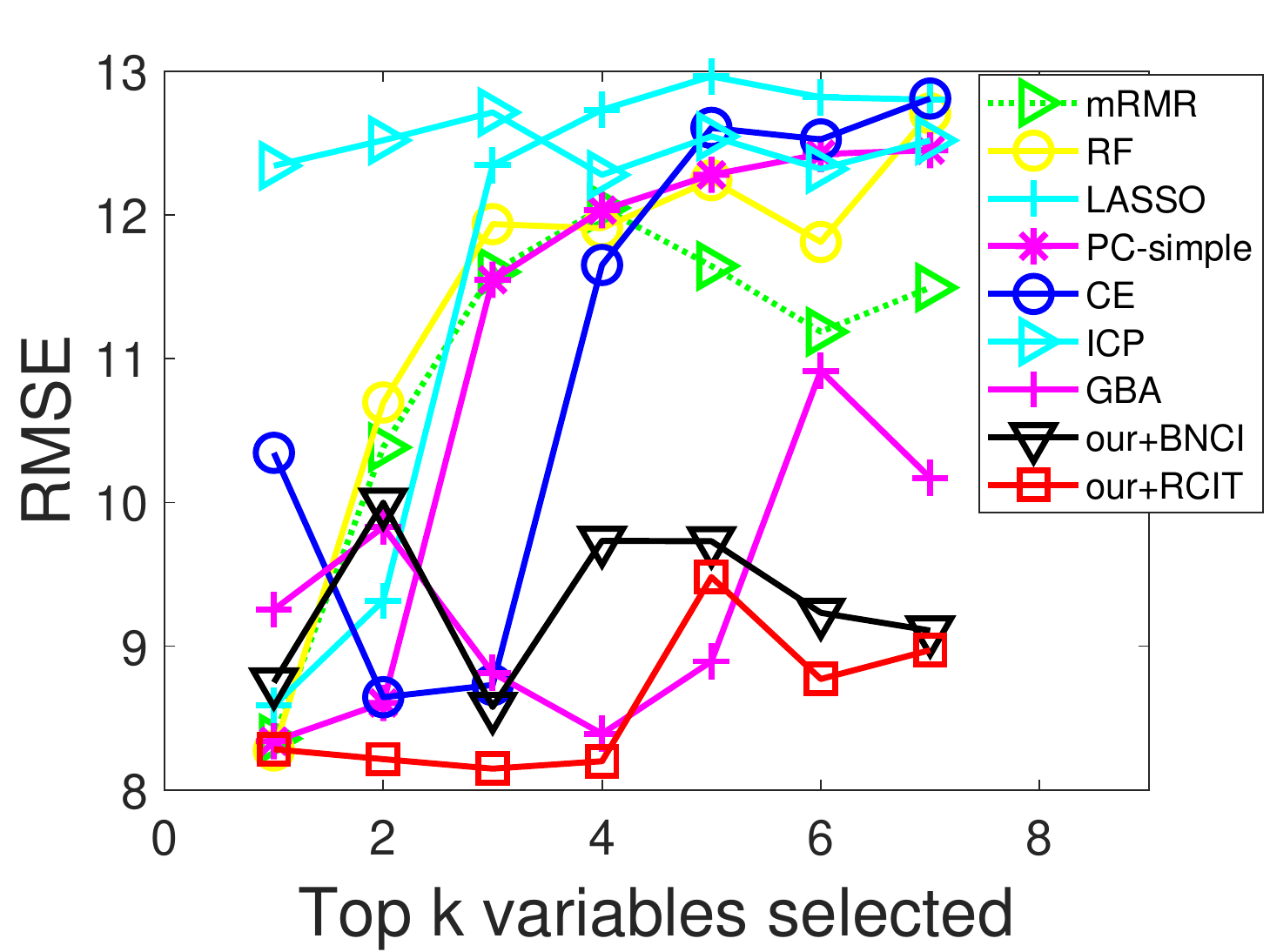}
}
\subfloat[RMSE on env. G3 \label{fig:real_RMSE_on_Env3}]{
  \includegraphics[width=1.3in]{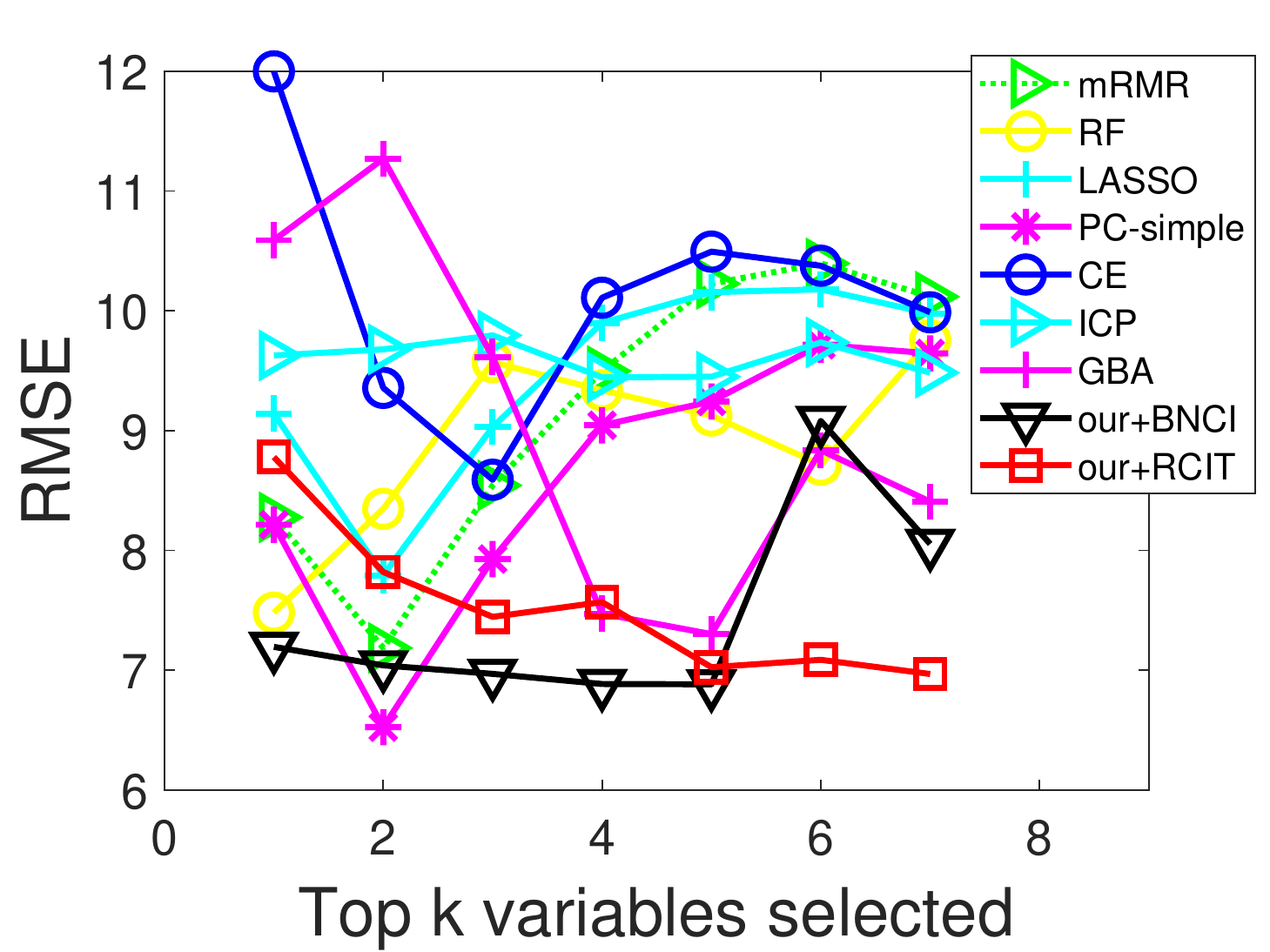}
}
\subfloat[RMSE on env. G4 \label{fig:real_RMSE_on_Env4}]{
  \includegraphics[width=1.3in]{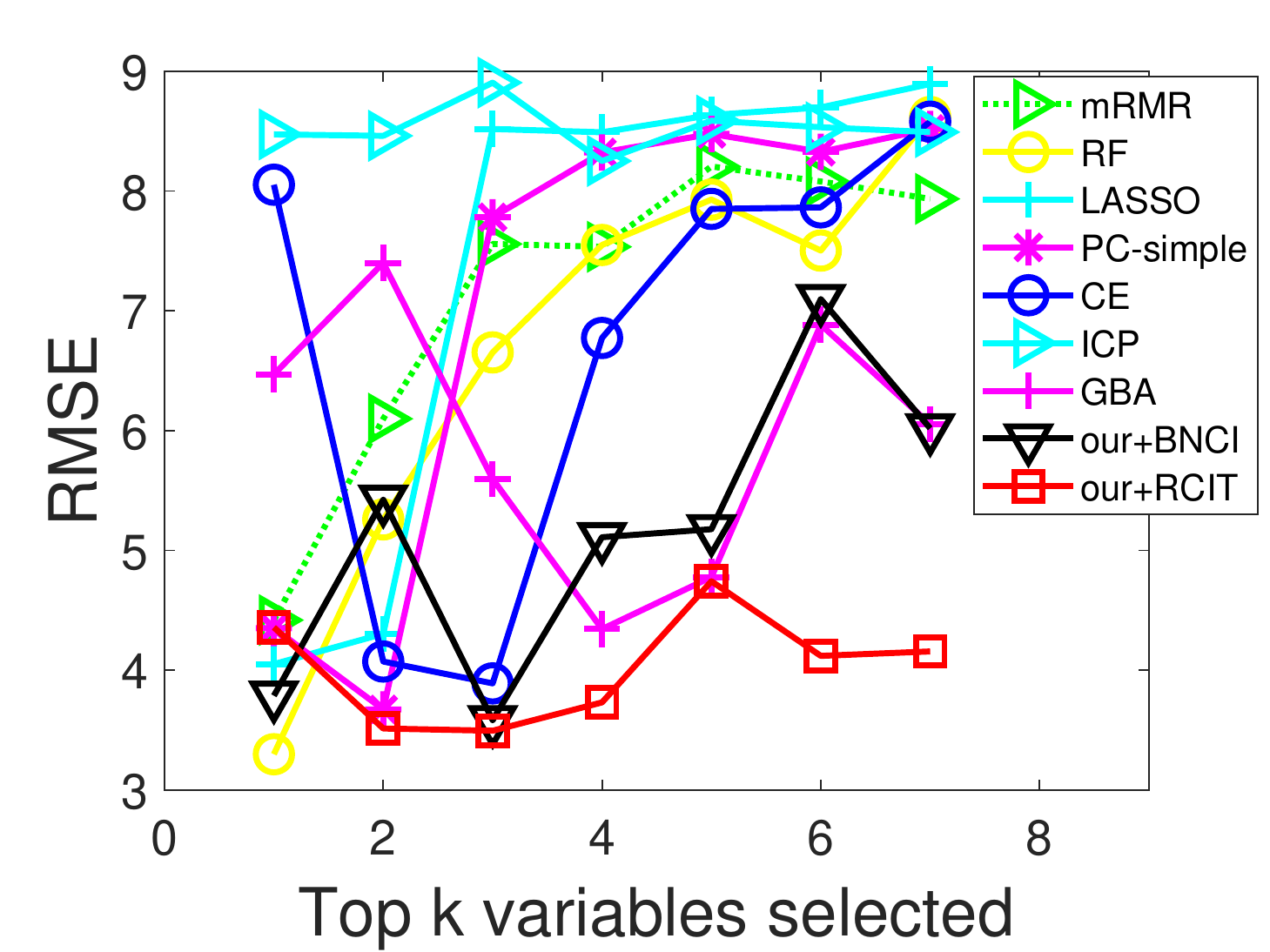}
}
\caption{Results of RMSE with top-$k$ selected variables on different environments. All algorithms are trained with data from environment G1, but tested on the data from each environment. When the test environment is different from the training one (e.g., G2, G3, and G4), our algorithm achieves better performance than baselines.}
\label{fig:results_real}
\end{figure*}

\par \textbf{Experimental Settings.}
In our experiments, we set the motor UPDRS scoring as the response variables $Y$.
To test the stability of all methods, we generate different environments by biased data separation based on different patients. Specifically, we separate the whole 42 patients into 4 patients' groups, including group 1 (G1) with recordings from 21 patients, and other three groups (G2, G3 and G4) are all with recordings from different 7 patients, where the different groups correspond to different environments. Considering a practical setting where a researcher has a single data set and wishes to train a model that can then be applied to other environments, in our experiments, we trained all models with data from environment G1, but tested them on all 4 groups.

\par \textbf{Experimental Results.}
We report the experimental results of RMSE with top-$k$ ranked variables in Figure \ref{fig:results_real}. Fig. \ref{fig:real_RMSE_on_Env1} shows that correlation based methods (LASSO, mRMR and RF) outperform causation based methods (GBA and our method), this is because the training and test have the similar distribution on env. G1, hence the spurious correlation between non-causal variables and response variable can bring positive power for prediction.
Moreover, we find ICP method achieves good performance in env. G1 since it cannot differentiate the spurious correlation from only one training environment.
Fig. \ref{fig:real_RMSE_on_Env2}, \ref{fig:real_RMSE_on_Env3} and \ref{fig:real_RMSE_on_Env4} demonstrate that causation based methods are better than correlation based methods when the test distributions are out of the training one, and our method, especially the method ``our+RCIT'', can almost achieve the best performance. The main reason is that spurious correlation on training could be different on testing, while causation based methods could discover causal variables for more stable prediction across environments, and our method performs the best on causal variables ranking and separation.
In addition, we observed that in non-i.i.d settings\footnote{The test distribution is different from the training one.}, the prediction performance might seriously decrease as inputting more selected variables, since some selected variables could be spuriously correlated with the response and unstable across environments.



\section{Conclusion}

In this paper, we focus on the problem of stable prediction via leveraging a seed variable for causal variable separation.
We argue that most of traditional prediction methods and variable selection methods are correlation based, resulting in instability problem on prediction across unknown environments.
By assuming that the casual variables and non-causal variables are independent, in this paper, we proposed a causal variable separation algorithm with a single CI test per variable, and provide a series of theorems to prove that our algorithm can precisely separate the causal variables. We also demonstrate that the precisely separated causal variables from our algorithm can bring stable prediction across unknown test data.
The experimental results on both synthetic and real-world datasets show that our algorithm outperforms the baselines for causal variables separation and stable prediction.

\bibliographystyle{plainnat}
\bibliography{paper}

\end{document}